\documentclass{article}

\usepackage[preprint]{neurips_2026}

\usepackage[utf8]{inputenc} % allow utf-8 input
\usepackage[T1]{fontenc}    % use 8-bit T1 fonts
\usepackage{hyperref}       % hyperlinks
\usepackage{url}            % simple URL typesetting
\usepackage{booktabs}       % professional-quality tables
\usepackage{amsfonts}       % blackboard math symbols
\usepackage{nicefrac}       % compact symbols for 1/2, etc.
\usepackage{microtype}      % microtypography
\usepackage{xcolor}         % colors

\usepackage{subcaption}
\usepackage{multirow}
\usepackage{amsmath}
\usepackage{amssymb}
\usepackage{mathtools}
\usepackage{amsthm}

\usepackage[capitalize,noabbrev]{cleveref}

\theoremstyle{plain}
\newtheorem{theorem}{Theorem}[section]

\newtheorem{lemma}[theorem]{Lemma}
\newtheorem{corollary}[theorem]{Corollary}
\theoremstyle{definition}
\newtheorem{definition}[theorem]{Definition}

\theoremstyle{remark}
\newtheorem{remark}[theorem]{Remark}

\title{Decoupling Motion and Geometry in 4D Gaussian Splatting}

\author{%
Yi Zhang\textsuperscript{1}\thanks{Equal contribution.} \quad
Yulei Kang\textsuperscript{1}\footnotemark[1] \quad
Jiangxin Sun\textsuperscript{2} \quad
Beihao Xia\textsuperscript{3} \quad
Jisheng Dang\textsuperscript{4} \quad
Jian-Fang Hu\textsuperscript{1}\thanks{Corresponding author.} \\
\textsuperscript{1} Sun Yat-sen University \quad 
\textsuperscript{2} University of Trento \\
\textsuperscript{3} Huazhong University of Science and Technology \quad 
\textsuperscript{4} Lanzhou University \\
}

\begin{document}

\maketitle

\begin{abstract}
  High-fidelity reconstruction of dynamic scenes is an important yet challenging problem. While recent 4D Gaussian Splatting (4DGS) has demonstrated the ability to model temporal dynamics, it couples Gaussian motion and geometric attributes within a single covariance formulation, which limits its expressiveness for complex motions and often leads to visual artifacts. To address this, we propose VeGaS, a novel velocity-based 4D Gaussian Splatting framework that decouples Gaussian motion and geometry. Specifically, we introduce a Galilean shearing matrix that explicitly incorporates time-varying velocity to flexibly model complex non-linear motions, while strictly isolating the effects of Gaussian motion from the geometry-related conditional Gaussian covariance. Furthermore, a Geometric Deformation Network is introduced to refine Gaussian shapes and orientations using spatio-temporal context and velocity cues, enhancing temporal geometric modeling. Extensive experiments on public datasets demonstrate that VeGaS achieves state-of-the-art performance.
\end{abstract}

\section{Introduction}

Photorealistic modeling of real-world dynamic scenes aims to synthesize images at arbitrary viewpoints and time instants. It remains a fundamental challenge in computer vision and machine learning, with broad applications in VR/AR, immersive gaming, and cinematic production. The difficulty stems from the diversity of real-world dynamics, ranging from rigid-body motion, where geometry should remain invariant under transformations, to non-rigid deformations, where both motion and geometry evolve over time under distinct physical constraints.

Neural Radiance Fields (NeRF) \cite{mildenhall2021nerf} and its variants \cite{chen2022tensorf, sun2022direct, hu2022efficientnerf, muller2022instant, fridovich2022plenoxels, takikawa2021neural, xu2022point} pioneered implicit functional representations for reconstructing complex scenes. More recently, 3D Gaussian Splatting (3DGS) \cite{kerbl20233d} has emerged as a compelling alternative that achieves real-time rendering by explicitly modeling scene as a set of discrete 3D Gaussian primitives. Despite their success, both paradigms are largely tailored to static scenes and lack native mechanisms for representing temporal dynamics.

To address this, 4DGaussians \cite{wu20244d} introduce a deformation field network to jointly learn the position offsets and geometric deformations of Gaussian points at each timestamp. 4DGS and its variants \cite{yang2023real, gao20257dgs} further improved the modeling quality by proposing a 4D Gaussian Splatting representation that integrates temporal dynamics into the 3D Gaussian primitives. This representation computes both the spatial position and geometry attributes (i.e., shape and orientation) of a Gaussian at a given timestamp through the conditional distribution of the 4D Gaussian covariance, yielding a constant-velocity linear motion and a time-independent geometric model. However, the time-invariant properties of its velocity and geometric representation limit its ability to capture complex non-linear motion and expressiveness during inference. Furthermore, due to the Gaussian covariance modeling approach, the optimization of Gaussian motion and geometric parameters becomes coupled. This affects geometric modeling during complex motion fitting, making the model prone to artifacts, as illustrated in Fig.~\ref{intro}.

\begin{figure}[t]
\centering
\includegraphics[width=\linewidth]{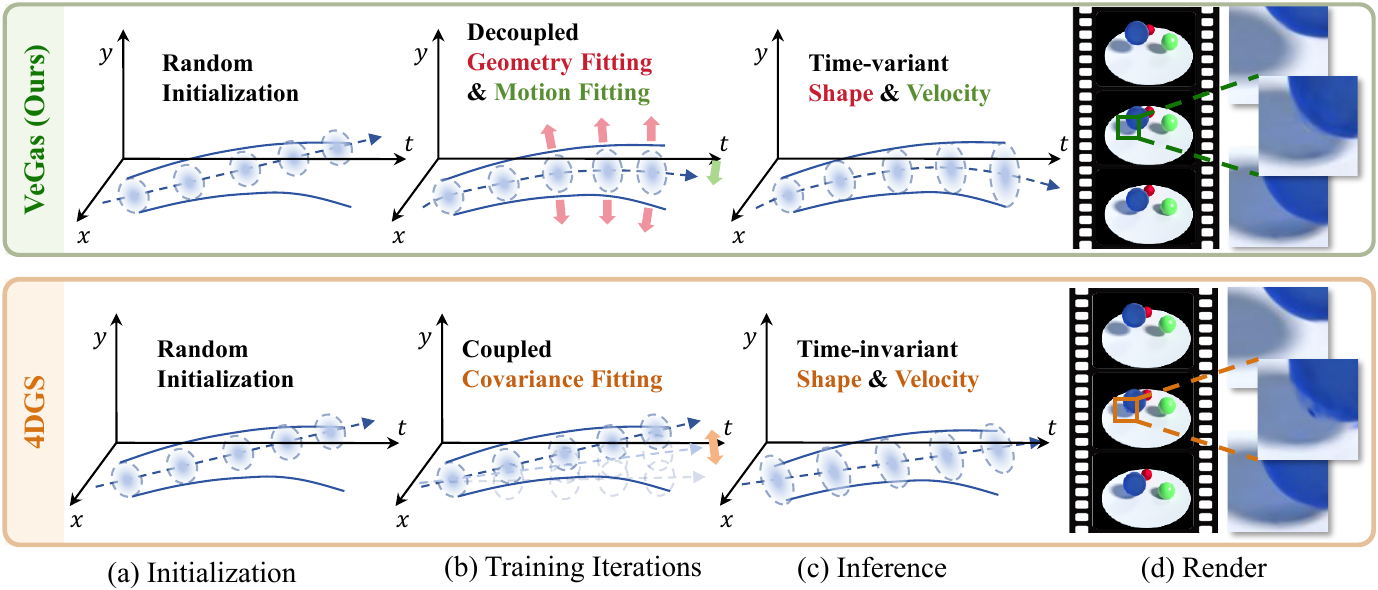}
\caption{
Graphic illustration of the proposed VeGaS in comparison with 4DGS. (a) Both methods share the same random initialization. (b) During training, VeGaS fits trajectories (as illustrated by the blue curves) using a decoupled motion and geometry formulation, whereas 4DGS adopts a coupled modeling scheme.
(c) At inference, VeGaS employs time-varying velocity and geometry to capture complex trajectories and deformations, while 4DGS assumes constant velocity and time-invariant geometry.
(d) Qualitative results show that VeGaS yields higher rendering fidelity, whereas 4DGS exhibits noticeable artifacts.
}
\label{intro}
\vspace{-5mm}
\end{figure}

In this paper, we present VeGaS (\textbf{Ve}locity-based Decoupling of Motion and Geometry in 4D \textbf{Ga}ussian \textbf{S}platting), a novel framework that decouples motion and geometry in 4D Gaussian Splatting to enhance dynamic scene rendering. Drawing inspiration from Galilean transformations, we propose a Galilean shearing matrix incorporating time-varying velocity for flexible modeling of complex non-linear motions, which naturally integrates into the original Gaussian covariance through congruence transformation. The transformed covariance inherently decouples the effects of Gaussian motion and geometric modeling, preventing interference with geometric modeling when fitting complex trajectories. Additionally, we present a lightweight network that independently refines the shape and orientation of the Gaussians over time. By decoupling the optimization of motion and geometry, our framework overcomes the limitations of previous approaches, providing a more expressive and reliable solution for dynamic scene reconstruction. Our contributions are summarized as follows:
\begin{itemize}
\item We propose a novel decoupled motion and geometry framework VeGaS which effectively addresses the artifact issues arising from covariance coupling in 4D Gaussian Splatting modeling.
\item We introduce a novel Gaussian motion modeling approach by incorporating time-variant velocity into the 4DGS representation and propose a deformation network to model time-varying Gaussian geometry, enhancing 4DGS expressiveness.
\item We conduct extensive experiments on public datasets, demonstrating that our method consistently achieves state-of-the-art results in both visual quality and quantitative performance.
\end{itemize}

\section{Related Works}

\textbf{Static Scene Novel View Synthesis.} 
Novel view synthesis is a crucial and challenging task in machine learning and computer vision. Neural Radiance Fields (NeRF) \cite{mildenhall2021nerf}, a pioneering work in this field, introduces an implicit scene representation that models color and density using multilayer perceptrons (MLP), delivering high-quality rendering results. Subsequently, numerous research has emerged aimed at improving the training and rendering efficiency of vanilla NeRF, employing techniques such as compact data representations \cite{chen2022tensorf, sun2022direct, hu2022efficientnerf, muller2022instant, fridovich2022plenoxels, takikawa2021neural, xu2022point}, or compressing neural networks \cite{gordon2023quantizing, reiser2021kilonerf}. Other works \cite{barron2021mip, barron2022mip, barron2023zip, verbin2024ref} focus on improving vanilla NeRF by mitigating aliasing artifacts or enhancing surface reflections, thereby boosting overall rendering performance. Recently, 3D Gaussian Splatting (3DGS) \cite{kerbl20233d} has made significant advances in scene reconstruction due to its fast rendering. This explicit modeling approach offers greater flexibility and controllability, motivating further research \cite{zhu2025rethinking, ye2024gaussian, guo2024semantic} to explore the application of 3DGS in areas such as 3D semantic segmentation and scene editing.

\textbf{Dynamic Scene Novel View Synthesis.}
Generating novel views of a dynamic scene from a series of captured 2D images presents a greater challenge compared to static scenes. Many works have extended NeRF-based methods to dynamic scenes. Some methods \cite{pumarola2021d, park2021nerfies, park2021hypernerf} incorporate time as a conditioning variable and learn deformation fields that warp points from a canonical space to their corresponding positions at each time step. NeRFPlayer \cite{song2023nerfplayer} further decomposes the 4D spatiotemporal space into static, deforming, and new regions, and introduces a hybrid feature streaming scheme for efficient neural field modeling. To improve efficiency, HexPlane \cite{cao2023hexplane} explicitly represents dynamic scenes using six learned feature planes, significantly reducing training time. Despite these advances, achieving real-time rendering with NeRF for scenes involving complex dynamics and view-dependent effects remains a significant challenge. Consequently, several recent works \cite{luiten2024dynamic, huang2024sc, lin2024gaussian, lu20243d, yang2024deformable} have explored extending 3DGS to dynamic scenes. Wu et al. \cite{wu20244d} use a deformation field network to capture the Gaussian deformations in position, rotation, and scaling, enabling accurate Gaussian transformations over time. 4DGS \cite{yang2023real} incorporates temporal dynamics by extending the Gaussians to a 4D representation that combines spatial and temporal dimensions. Although these Gaussian-based methods effectively capture the temporal dynamics in scene modeling, they do not comprehensively account for the variable motion dynamics of Gaussian points.

\section{Method}

\subsection{Preliminary}
\textbf{3D Gaussian Splatting.} 
3D Gaussian Splatting \cite{kerbl20233d} explicitly renders static scenes using a set of anisotropic 3D Gaussian distributions. Each Gaussian primitive is defined by its center position $\mathbf{\mu} \in \mathbb{R}^3$ and covariance matrix $\Sigma \in \mathbb{R}^{3 \times 3}$. To ensure the semi-positive definiteness of the covariance matrix and simplify the optimization process, the covariance matrix $\mathbf{\Sigma}$ is decomposed into a rotation matrix $\mathbf{R}$ and a scaling matrix $\mathbf{S}= \mathrm{diag}(s_x, s_y, s_z)$:
\begin{equation}
\label{equation_1}
\mathbf{\Sigma} = \mathbf{R} \mathbf{S} \mathbf{S}^T \mathbf{R}^T.
\end{equation}
In the 3D Gaussian representation, a set of spherical harmonic ($SH$) coefficients are also employed to represent view-dependent color, along with an opacity $\mathbf{\alpha} \in [0,1]$.

\textbf{4D Gaussian Splatting.} To render dynamic scenes, 4D Gaussian Splatting (4DGS) \cite{yang2023real} reformulate the center position, covariance matrix, and Gaussian color in 3D Gaussian as follows: (1) Extending Gaussian center position with temporal position as $\mathbf{\mu} = (\mathbf{\mu}_x, \mathbf{\mu}_y, \mathbf{\mu}_z, \mathbf{\mu}_t) \in \mathbb{R}^4$; (2) Re-formulating Covariance matrix as $\mathbf{\Sigma} = \mathbf{R} \mathbf{S} \mathbf{S}^T \mathbf{R}^T\in \mathbb{R}^{4 \times 4}$, where $\mathbf{S}= diag(s_x,s_y,s_z,s_t)$ is a scaling matrix and $\mathbf{R}\in \mathbb{R}^{4 \times 4}$ is a rotation matrix; (3) Representing the Gaussian color by spherical harmonic coefficients $\mathbf{h} \in \mathbb{R}^{3{(k_v+1)^2(k_t+1)}}$, where $k_v$ is the view degrees of freedom, and $k_t$ indicates the time degree of freedom. At each time step, the 4D Gaussian induces a conditional 3D Gaussian distribution, which can be expressed as:
\begin{equation}
\label{equation_2}
\begin{aligned}
\mathbf{\mu} _{xyz \mid t} &= \mu _{1:3} +(t - \mu_t) \mathbf{\Sigma}_{1:3,4} {\mathbf{\Sigma}_{4,4}}^{-1} , \\
\mathbf{\Sigma}_{xyz \mid t} &= \mathbf{\Sigma}_{1:3,1:3} - \mathbf{\Sigma}_{1:3,4} {\mathbf{\Sigma}_{4,4}}^{-1} \mathbf{\Sigma}_{4,1:3},
\end{aligned}
\end{equation}
where the 3D center $\mu_{xyz \mid t}$ increases linearly over time with a constant velocity of $\mathbf{\Sigma}_{1:3,4}\mathbf{\Sigma}_{4,4}^{-1}$. The induced 3D covariance $\mathbf{\Sigma}_{xyz \mid t}$ is time-invariant, indicating that the 4D Gaussian’s spatial shape and orientation do not depend on the temporal variable. 

Although 4DGS supports dynamic scene reconstruction, its constant-velocity motion assumption and time-invariant geometric attributes (i.e., geometry is independent of time) together limit its capacity to represent complex motions and time-varying geometric deformations. Moreover, $\mathbf{\Sigma}$ conflates geometric structure and motion dynamics into a single parameterization, coupling their updates during optimization rather than allowing them to be tuned independently. As a result, the model struggles to disentangle the temporal evolution of geometry and motion, which can degrade reconstruction quality and introduce artifacts.

\begin{figure}[htbp]
\centering
\includegraphics[width=\textwidth]{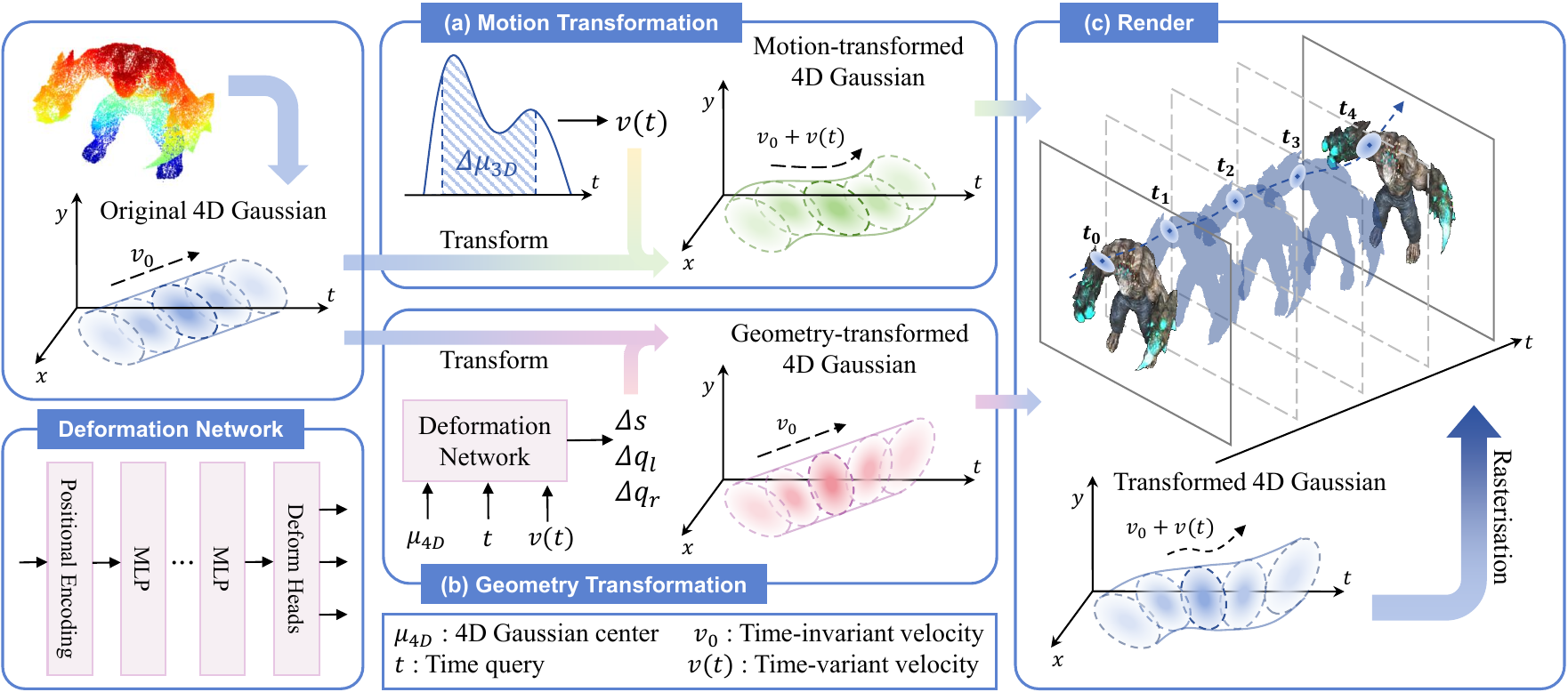}
\caption{
Overview of our proposed velocity-based decoupling of motion and geometry for 4D Gaussian Splatting. (a) Original velocity $v_0$ is transformed using the time-varying velocity $v(t)$ calculated from the shearing matrix, enabling original 4D Gaussians to move along non-linear trajectories in continuous time. (b) The deformation network predicts geometric transformations of Gaussians at any time based on velocity information, time query $t$, and 4D Gaussian center $\mu_{4D}$. (c) Combining the velocity and geometry transformations, the rendered images are obtained through differentiable rasterization of transformed Gaussians at each frame.
}
\label{primary}
\vspace{-6mm}
\end{figure}

\subsection{Motion-Geometric Decoupled Representation}
%This decoupling prevents the temporal evolution process from conflating complex motion flows with local shape deformations, resulting in more robust synthesis.
To overcome the limitations of 4DGS, a natural and intuitive solution is to introduce a time-varying velocity that enables non-linear trajectories, thereby enhancing the flexibility of 4D Gaussians in representing complex dynamics. Crucially, this motion enhancement must be strictly decoupled from geometric modeling, such that the velocity influences only the spatial position while preserving the intrinsic 3D shape and orientation of each Gaussian. We achieve this by introducing a motion–geometry decoupled representation, whose formulation is grounded in a Galilean shearing analysis.

\textbf{Theoretical Analysis.}
To model the motion of Gaussian points over time, we draw inspiration from the Galilean transformation in classical mechanics, using a spatiotemporal shearing operation to drag points along their trajectories.

\begin{definition}
[Galilean Shearing]
\label{def:galilean_shearing}
In a 4D spatio-temporal continuum $(x, y, z, t)$, any linear transformation that imparts a constant velocity $ \mathbf{v} \in \mathbb{R}^3$ to a point while preserving the absolute temporal coordinate ($t'=t$) is equivalent to a Galilean transformation. This is represented by a shearing matrix $\mathbf{V}$ acting on the 4D coordinates:
\begin{equation}
\begin{pmatrix} \mathbf{x}' \\ t' \end{pmatrix} =  \mathbf{V} \begin{pmatrix} \mathbf{x} \\ t \end{pmatrix} = \begin{pmatrix} \mathbf{I}_3 & \mathbf{v} \\ \mathbf{0} & 1 \end{pmatrix} \begin{pmatrix} \mathbf{x} \\ t \end{pmatrix},
\end{equation}
where $\mathbf{I}_3$ denotes the $3\times3$ identity matrix.
\end{definition}

\begin{remark}
Physically, matrix $\mathbf{V}$ modulates trajectories by mapping the static temporal axis to a slanted trajectory in spacetime, where the slope corresponds to $\mathbf{v}$. Since $\det(\mathbf{V})=1$, this transformation preserves the 4D Gaussian volume and conserves the total probability mass of the Gaussian density function, ensuring the feasibility of applying the shear transformation to the Gaussian representation.
\end{remark}

While the above definition provides a method to introduce velocity through shearing, it is essential to demonstrate that such a 4D transformation does not distort the rendered 3D geometry at any given time. The intrinsic 3D geometry of a 4D Gaussian at any time instance $t$ is determined by its conditional distribution $P(\mathbf{x}|t)$. The covariance of this distribution, which characterizes the 3D ellipsoid's shape and orientation, is given by the Schur complement of the temporal block within the joint 4D covariance matrix $\mathbf{\Sigma}$. This allows us to prove the stability of the 3D geometry using the principle of Schur complement stability.

\begin{theorem}
[Schur Complement Invariance]
\label{thm:schur_invariance}
Let $\mathbf{\Sigma} \in \mathbb{R}^{4 \times 4}$ be a symmetric positive semi-definite covariance matrix of a 4D Gaussian, and let $\mathbf{\Sigma}' = \mathbf{V}\mathbf{\Sigma}\mathbf{V}^\top$ be the congruence transformation induced by the shearing matrix $\mathbf{V}$. Denoting by $\text{Schur}_{4,4}(\cdot)$, the Schur complement with respect to the temporal dimension, the following invariance holds:
\begin{equation}
\text{Schur}_{4,4}(\mathbf{\Sigma}') = \text{Schur}_{4,4}(\mathbf{\Sigma}).
\end{equation}
\end{theorem}

\cref{thm:schur_invariance} (proof details in Appendix \ref{appendix:proof_schur}) ensures that while the Gaussian leans in 4D spacetime to represent a velocity vector $\mathbf{v}$, its 3D cross-section at any given time $t$ retains the same shape, scale, and orientation as its rest state. Although the Galilean transformation is originally defined for constant velocity, its shearing formulation can be naturally extended to accommodate time-varying velocities, enabling the modeling of complex, non-linear trajectories.

\textbf{Shearing-based Motion Modeling.} Building on the above analysis, we extend the velocity shearing matrix $\mathbf{V}$ to incorporate a time-varying instantaneous velocity $\mathbf{v}(t) = (v_x, v_y, v_z)^\top$ as follows:
\begin{equation}
\mathbf{V} = \begin{pmatrix}
\mathbf{I}_3 & \mathbf{v}(t) \\
\mathbf{0} & 1
\end{pmatrix}.
\end{equation}

Applying the shearing matrix $\mathbf{V}$ to the original covariance $\mathbf{\Sigma}$, we construct the velocity-aware 4D covariance matrix $\mathbf{\Sigma}'$ via a congruence transformation, which preserves symmetry and positive semi-definiteness (proof in Appendix \ref{appendix:spsd_proof}):
\begin{equation}
\mathbf{\Sigma}' = \mathbf{V} \mathbf{R} \mathbf{S} \mathbf{S}^T \mathbf{R}^T \mathbf{V}^T = \mathbf{V} \mathbf{\Sigma} \mathbf{V}^T.
\end{equation}

By performing block-wise partitioning on the transformed covariance matrix $\mathbf{\Sigma}'$ and leveraging the properties of the multivariate Gaussian distribution, we derive the induced conditional 3D Gaussian distribution at time $t$ (detailed calculation in Appendix \ref{appendix:derivation_conditional_distribution}). The corresponding conditional mean $\mu_{xyz \mid t}'$ and covariance matrix $\mathbf{\Sigma}_{xyz \mid t}'$ are given by:
\begin{equation}
\begin{aligned}
\label{eqMiuSigVt}
\mu_{xyz \mid t}' &= \mu_{1:3} + \left( \mathbf{\Sigma}_{1:3,4}\mathbf{\Sigma}_{4,4}^{-1} + \mathbf{v}(t) \right)(t - \mu_t), \\    
\mathbf{\Sigma}_{xyz \mid t}' &= \mathbf{\Sigma}_{1:3,1:3} - \mathbf{\Sigma}_{1:3,4}\mathbf{\Sigma}_{4,4}^{-1}\mathbf{\Sigma}_{4,1:3},
\end{aligned}
\end{equation}
where the time-varying velocity $\mathbf{v}(t)$ explicitly models the non-linear motion evolution of Gaussian point trajectories over time. By comparing Eq.~(\ref{equation_2}) with Eq.~(\ref{eqMiuSigVt}), we observe that the conditional covariance $\mathbf{\Sigma}_{xyz \mid t}'$ is identical to that in the original 4DGS formulation and remains independent of $\mathbf{v}(t)$. This indicates that introducing the shearing matrix $\mathbf{V}$ affects only the trajectory of the Gaussian center, while preserving its intrinsic 3D shape and orientation, which is consistent with the theoretical analysis \cref{thm:schur_invariance}. 

\textbf{Non-linear Trajectory Integration.}
Let $\mathbf{v}_0 = \mathbf{\Sigma}_{1:3,4}\mathbf{\Sigma}_{4,4}^{-1}$ denote the time-invariant velocity component. The cumulative displacement $\Delta\mu_{3D}$ relative to the spatial mean $\mu_{1:3}$ can be calculated by integrating the total instantaneous velocity over time:
\begin{equation}
\Delta\mu_{3D} = \int_{\mu_t}^{t} (\mathbf{v}(\tau) + \mathbf{v}_0)d\tau = \int_{\mu_t}^{t} \mathbf{v}(\tau)d\tau + \mathbf{v}_0(t-\mu_t).
\end{equation}

To handle non-linear trajectories beyond simple constant velocity, we model the time-variant velocity $\mathbf{v}(\tau)$ as a continuous function parameterized by a set of $N_v$ equidistantly sampled velocity anchors across the temporal domain $T$. Each anchor is associated with a learnable velocity vector, serving as a nodal point for motion optimization. For any query time $t$, the instantaneous velocity $\mathbf{v}(t)$ is obtained via linear interpolation between the two temporally adjacent velocity anchors. 

We employ an efficient segmented numerical integration scheme based on the temporal distance between $t$ and $\mu_t$: (1) \textit{Intra-anchor Interpolation:} When $t$ and $\mu_t$ reside within the same anchor interval, the displacement is calculated as the area of a single trapezoid formed by $\mathbf{v}(t)$ and $\mathbf{v}(\mu_t)$. (2) \textit{Cross-anchor Accumulation:} For intervals spanning multiple anchors, the integral is decomposed into a left boundary trapezoid (from $\mu_t$ to the next anchor), a right boundary trapezoid (from the last anchor to $t$), and a series of interior segments representing full inter-anchor intervals.

To avoid redundant computations for the interior segments, we utilize a prefix sum of the anchor velocities. This allows the cumulative displacement of any number of full intervals to be computed in $\mathcal{O}(1)$ time. Let $t_k$ denote the timestamp of the $k$-th anchor and $m$ represent the number of complete intervals. The cumulative displacement is computed as:
\begin{equation}
\int_{t_k}^{t_{k+m}} \mathbf{v}(\tau) d\tau = \sum_{i=k}^{k+m-1} \frac{\mathbf{v}(t_{i}) + \mathbf{v}(t_{i+1})}{2} \cdot \Delta t,
\end{equation}
where $\Delta t = \frac{T}{N_v-1}$ is the constant temporal stride.

\subsection{Geometric Deformation Network}
While the shearing matrix $\mathbf{V}$ enhances the flexibility of Gaussian motion, complex dynamic scenes often involve high-frequency geometric deformations (e.g., non-rigid muscle movement or clothing wrinkles) that are more challenging to model. Unlike motion, which is modeled using a velocity-parameterized covariance, we introduce a lightweight deformation network to capture the time-varying geometric attributes of the Gaussians for more accurate deformation representation.

The network $\mathcal{F}_\theta$ takes the spatio-temporal context and a condition $t$ as input, predicting residuals for scaling, rotation, and position. To provide the network with explicit motion cues, we also incorporate the velocity. Specifically, the deformation network is formulated as:
\begin{equation}
    \Delta\mathbf{s},\, \Delta\mathbf{q},\, \Delta\mathbf{q}_r = \mathcal{F}_\theta\big( \gamma(\mathbf{\mu_{3D}}),\, \gamma(\mu_t),\, \gamma(t_q),\, \gamma(\mathcal{V}) \big),
\end{equation}
where $\mathbf{\mu_{3D}} \in \mathbb{R}^3$ denotes the canonical 3D Gaussian center, $\mu_t \in \mathbb{R}$ is the Gaussian temporal mean, $t_q \in \mathbb{R}$ is the target query time, and $\mathcal{V} \in \mathbb{R}^{N_v \times 3}$ represents the motion velocity feature, formed by concatenating the velocity vectors from the $N_v$ anchors. To capture high-frequency variations, each input is mapped to a higher-dimensional space via positional encoding $\gamma(\cdot)$:
\begin{equation}
    \gamma(\mathbf{x}) = \big[ \mathbf{x},\; \sin(2^i \mathbf{x}),\; \cos(2^i \mathbf{x}) \big]_{i=0}^{L_{\mathbf{x}}-1},
\end{equation}
where $L_{\mathbf{x}}$ denotes the feature-specific frequency bands. The encoded features are concatenated and fed into the MLP-based blocks of $\mathcal{F}_\theta$, which comprises successive linear layers followed by Layer Normalization and ReLU activations. The network outputs residuals $\Delta\mathbf{s} \in \mathbb{R}^4$ for scaling, and $\Delta\mathbf{q}, \Delta\mathbf{q}_r \in \mathbb{R}^4$ as quaternion-based rotation. The deformed attributes are then updated as:
\begin{equation}
\mathbf{S}' = \mathbf{S} +  diag(\Delta\mathbf{s}), 
\mathbf{q}' = \mathbf{q} \otimes \Delta\mathbf{q}, 
\mathbf{q}_r' = \mathbf{q}_r \otimes \Delta\mathbf{q}_r,
\end{equation}
where $\otimes$ denotes the quaternion multiplication. We then follow \cite{yang2023real} and employ a dual-quaternion calculation strategy to construct rotations in 4D space. More specifically, the final 4D rotation $\mathbf{R}'$ is then constructed from the refined quaternions $\mathbf{q}'$ and $\mathbf{q}_r'$ as $\mathbf{R}' = \mathbf{q}' \mathbf{q}_r'$. %Finally, the covariance that encodes the geometries can be calculated as $\mathbf{R}' \mathbf{S}' \mathbf{S}'^T \mathbf{R}'^T$.

\subsection{Loss Function}
Our system is optimized by minimizing the following loss:
\begin{equation}
L = \lambda L_{L1} + (1 - \lambda)L_{\mathrm{D-SSIM}},
\end{equation}
where $L_{L1}=\| \mathbf{I} - \mathbf{I}_{\mathrm{gt}} \|_{L_1}$ is employed to quantify the pixel-wise difference between the ground truth image $\mathbf{I}_{\mathrm{gt}}$ and the corresponding rendering $\mathbf{I}$. $L_{\mathrm{D-SSIM}}=1 - \mathrm{SSIM}(\mathbf{I}, \mathbf{I}_{\mathrm{gt}})$ evaluates the perceptual quality of the rendered image relative to the ground truth.

\section{Experiments}

\subsection{Experimental Setup}
% \textbf{Datasets.} We evaluate our method on two representative benchmarks: (1) \textit{D-NeRF Dataset} \cite{pumarola2021d}: A monocular dataset consists of videos from eight synthetic scenes. To evaluate the model's performance, we follow the previous approaches \cite{fridovich2023k, yang2023real} and use the standard
% testing protocol, where the views used for model training and testing are different.
% (2) \textit{Neural 3D Video Dataset} (Neu3DV) \cite{li2022neural}: A real-world dataset includes six multi-view video scenes captured by 18 to 21 cameras. Consistent with existing approaches \cite{cao2023hexplane, fridovich2023k, yang2023real}, one viewpoint is reserved for testing, while the remaining viewpoints are used for training.
\textbf{Datasets.} We evaluate our method on two representative benchmarks: (1) \textit{D-NeRF Dataset} \cite{pumarola2021d}: A monocular dataset consists of videos from 8 synthetic scenes. 
(2) \textit{Neural 3D Video Dataset} (Neu3DV) \cite{li2022neural}: A real-world dataset includes 6 multi-view video scenes captured by 18 to 21 cameras. 

\begin{table}[th]
\centering
\caption{Comparison with state-of-the-art methods on the Neu3DV dataset. $+$ denotes results reproduced under our experimental setting, as the original protocol adopts a different train/test split.}
{
\begin{tabular}{l|c|cccc}
\toprule
\multicolumn{1}{l|}{\textbf{Methods}} & \textbf{Venue}  & \textbf{PSNR} $\uparrow$ & \textbf{SSIM} $\uparrow$ & \textbf{LPIPS} $\downarrow$ \\
\midrule
Neural Volumes \cite{lombardi2019neural} & TOG     & 22.80 &   0.94   & 0.295 \\
LLFF \cite{mildenhall2019local}          & TOG     & 23.24 &   0.92   & 0.235 \\
DyNeRF \cite{li2022neural}               & CVPR    & 29.58 &   0.98   & 0.099 \\
HexPlane \cite{cao2023hexplane}          & CVPR    & 31.70 &   0.98   & 0.075 \\
K-Planes-explicit \cite{fridovich2023k}  & CVPR    & 30.88 &   0.98   &   -   \\
K-Planes-hybrid \cite{fridovich2023k}    & CVPR    & 31.63 &   0.98   &   -   \\
MixVoxels-L \cite{wang2023mixed}         & ICCV    & 30.80 &   0.98   & 0.126 \\
StreamRF \cite{li2022streaming}          & NeurIPS & 29.58 &     -    &   -   \\
NeRFPlayer \cite{song2023nerfplayer}     & TVCG    & 30.69 &   0.97   & 0.111 \\
HyperReel \cite{attal2023hyperreel}      & CVPR    & 31.10 &   0.96   & 0.096 \\
4DGaussians \cite{wu20244d}              & CVPR    & 31.02 &   0.97   & 0.150 \\
4DGS \cite{yang2023real}                 & ICLR    & 32.01 &   0.97   & 0.055 \\
FreeTimeGS$^+$ \cite{wang2025freetimegs}     & CVPR    & 32.12 &   0.97   & 0.049 \\
Ours                                     & -       & \textbf{32.68} & \textbf{0.98}  & \textbf{0.045} \\
\bottomrule
\end{tabular}
}
\label{tab:Neural_3D_Video}
\vspace{-3mm}
\end{table}

\begin{figure}[th]
\centerline{\includegraphics[width=0.9\textwidth]{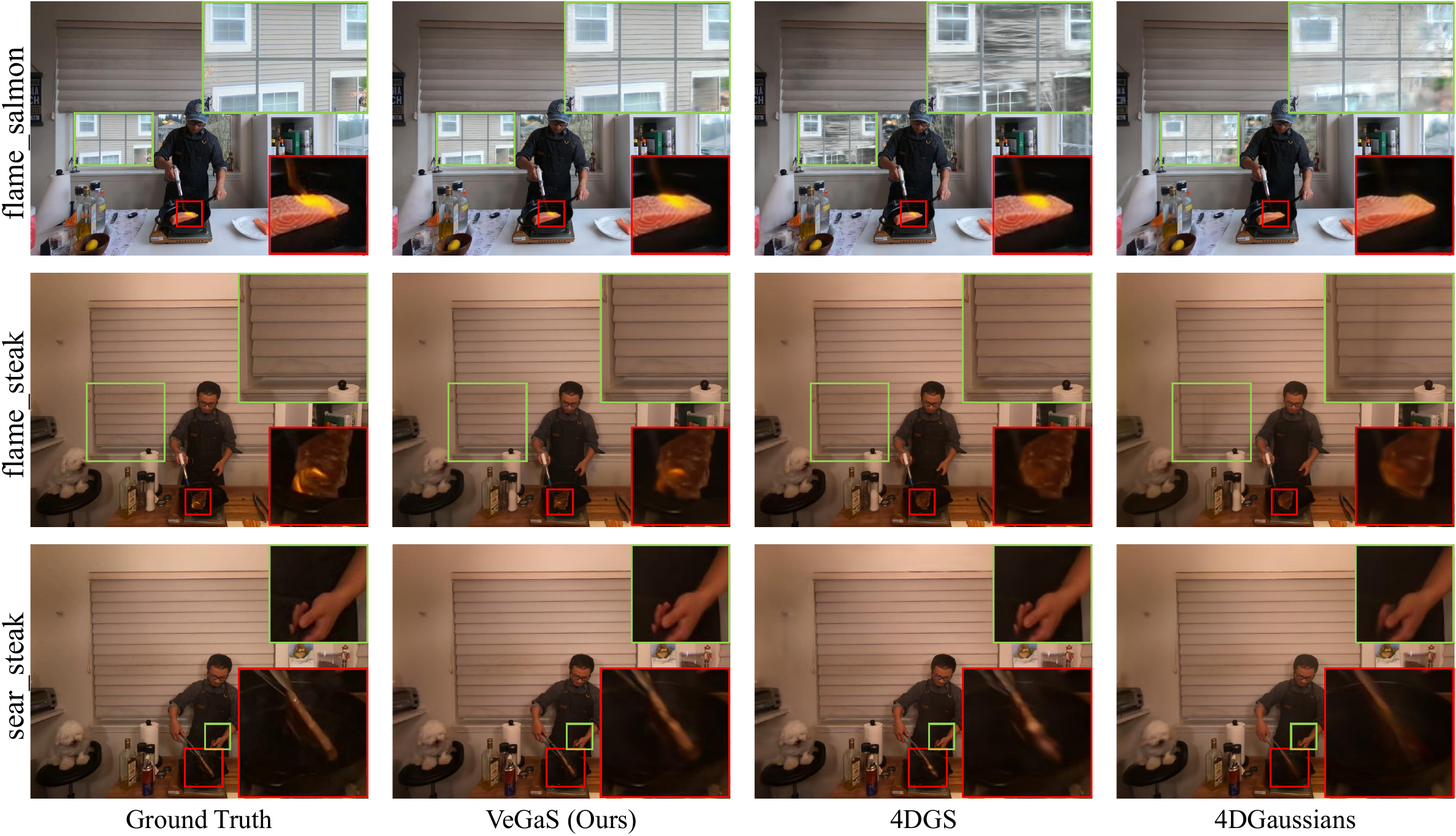}}
\caption{Qualitative result on the Neural 3D Video dataset. Our method exhibits noticeably higher visual quality compared to others.}
\label{real_quantitative_results}
\vspace{-7mm}
\end{figure}

% \textbf{Evaluation Metrics.} We evaluate the rendering quality using three widely adopted image quality metrics: Peak Signal-to-Noise Ratio (PSNR), Structural Similarity Index Measure (SSIM), and Learned Perceptual Image Patch Similarity (LPIPS) \cite{zhang2018unreasonable}.
\vspace{-1mm}
\textbf{Implementation Details.} Following \cite{yang2023real}, we use the Adam optimizer and adopt the same hyperparameter setting, including loss weight, learning rate, threshold and initialized number of Gaussians, to train our VeGaS model with a total of 30k iterations. The densification is terminated at the midpoint of the optimization schedule. A Gaussian filter with a threshold $p(t) < 0.05$ is employed to screen Gaussians for novel view rendering. For the time-variant velocity, we set the learning rate to $2e^{-3}$ and the number of anchors to 6. The learning rate of the deformation network decays from $8e^{-4}$ to $1.6e^{-6}$, with weight regularization of $1e^{-6}$ for training stability. We reproduce the results of FreeTimeGS \cite{wang2025freetimegs} on the Neu3DV dataset. All experiments are conducted on NVIDIA A6000 GPU.

% \textbf{Compared Approaches.} We compare our method with the current state-of-the-art approaches, including Neural Volumes \cite{lombardi2019neural}, LLFF \cite{mildenhall2019local}, DyNeRF \cite{li2022neural}, K-Planes \cite{fridovich2023k}, MixVoxels-L \cite{wang2023mixed}, StreamRF \cite{li2022streaming}, NeRFPlayer \cite{song2023nerfplayer}, HyperReel \cite{attal2023hyperreel}, 4DGaussians \cite{wu20244d}, 4DGS \cite{yang2023real}, 4DGV \cite{dai20254d}, D-NeRF \cite{pumarola2021d}, TiNeuVox \cite{fang2022fast}, V4D \cite{gan2023v4d}, and 7DGS \cite{gao20257dgs}. Following the evaluation protocol in \cite{yang2023real}, results on the Neu3DV and D-NeRF datasets are presented. For 4DGS, we reproduce the experimental results using their official released code. Specifically, for the Neu3DV dataset, Neural Volumes, LLFF, DyNeRF, and StreamRF only provide performance results for the \textit{flames salmon} scene. For NeRFPlayer and HyperReel, only SSIM is reported instead of MS-SSIM as with the other methods. Additionally, results for 4DGaussians are provided for the \textit{Spinach}, \textit{Beef}, and \textit{Steak} scenes. On the D-NeRF dataset, for 4DGaussians, rendering is performed at $800 \times 800$ resolution, with downsampling by a factor of $2$ for the other methods. 

\subsection{Results on Dynamic Scenes}
\vspace{-1.5mm}
\textbf{Results on multi-view real scenes.}
Table~\ref{tab:Neural_3D_Video} presents a quantitative comparison between our framework and state-of-the-art methods on the Neural 3D Video (Neu3DV) dataset, a multi-view real-world benchmark. As shown, our method consistently outperforms prior approaches across all evaluation metrics. In particular, compared with 4DGS, our approach improves PSNR from $32.01$ to $32.68$, achieving an absolute gain of $0.67$ dB. Meanwhile, LPIPS is reduced from $0.10$ to $0.09$, corresponding to a relative improvement of over $10\%$. The improved LPIPS scores indicate better preservation of fine-grained details and perceptual sharpness in the reconstructed scenes.

\begin{table}[t]
\centering
\caption{Quantitative comparison of methods on the monocular synthetic D-NeRF dataset.}
{
\begin{tabular}{l|c|cccc}
\toprule
\multicolumn{1}{l|}{\textbf{Methods}} & \textbf{Venue } & \textbf{PSNR} $\uparrow$ & \textbf{SSIM} $\uparrow$ & \textbf{LPIPS} $\downarrow$ \\
\midrule
T-NeRF \cite{pumarola2021d}               & CVPR     & 29.51 &   0.95   &  0.08 \\
D-NeRF \cite{pumarola2021d}               & CVPR     & 29.67 &   0.95   &  0.07 \\
TiNeuVox \cite{fang2022fast}              & SIGGRAPH Asia  & 32.67 &   0.97   &  0.04 \\
K-Planes-explicit \cite{fridovich2023k}   & CVPR     & 31.05 &   0.97   &   -   \\
K-Planes-hybrid \cite{fridovich2023k}     & CVPR     & 31.61 &   0.97   &   -   \\
V4D \cite{gan2023v4d}                     & TVCG     & 33.72 &   0.98   &  0.02 \\
4DGaussians \cite{wu20244d}               & CVPR     & 33.30 &   0.98   &  0.03 \\
4DGS \cite{yang2023real}                  & ICLR     & 34.09 &   0.98   &  0.02 \\
7DGS \cite{gao20257dgs}                   & ICCV     & 34.34 &   0.97   &  0.03 \\ 
Ours                                      & -        & \textbf{34.67} &   \textbf{0.99}   &  \textbf{0.02}  \\
\bottomrule
\end{tabular}
}
\label{tab:D-NeRF}
\vspace{-3mm}
\end{table}

\begin{figure}[t]
\centerline{\includegraphics[width=0.9\textwidth]{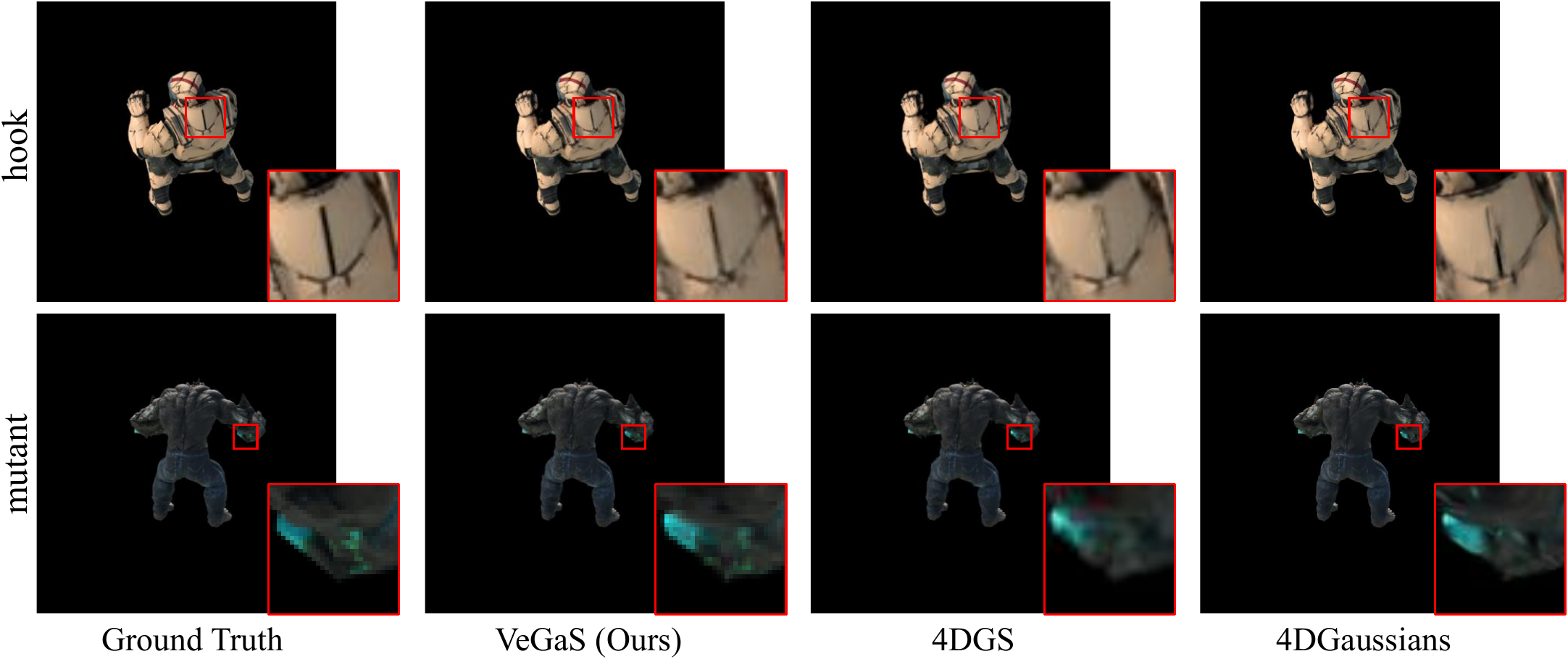}}
\caption{Qualitative results on the D-NeRF dataset, where our method captures finer details.}
\label{synthetic_quantitative_results}
\vspace{-7mm}
\end{figure}

Fig.\ref{real_quantitative_results} provides qualitative comparisons of novel view synthesis on the Neu3DV dataset. The visual results show that 4DGS suffers from noticeable artifacts, including distorted backgrounds outside the windows in the \textit{flame\_salmon} scene, as well as degraded cross-section textures of the steak in the \textit{sear\_steak} scene. These artifacts stem from both the coupling of Gaussian motion modeling and geometric attributes through the covariance matrix and the time-invariant nature of Gaussian properties, which limit accurate fitting of complex dynamics. 4DGaussians exhibits local blurring and a loss of fine-grained details, such as the background regions outside the windows in \textit{flame\_salmon}, the flames in \textit{flame\_salmon} and \textit{flame\_steak}, as well as the steak surface and finger regions in \textit{sear\_steak}. Compared to these methods, VeGaS consistently produces results with higher visual fidelity. Fine-grained details, including irregular flame patterns, clear outdoor scenery through windows, and well-defined finger structures are better preserved, indicating more faithful geometric reconstruction and sharper perceptual quality.

\vspace{-1.5mm}
\textbf{Results on monocular synthetic scenes.}
We evaluate our method on monocular dynamic scenes, which are widely regarded as challenging due to inherent ambiguity and limited information available from single-view observations. Tab.~\ref{tab:D-NeRF} presents a comparison between our framework and state-of-the-art methods on the D-NeRF dataset, a monocular synthetic benchmark. Our method achieves superior performance compared to all competing approaches across the evaluated metrics. These results demonstrate that our method can reliably exploit temporal coherence to compensate for the lack of multi-view constraints, achieving robust dynamic scene reconstruction under monocular settings.

Fig.~\ref{synthetic_quantitative_results} presents qualitative comparisons on the D-NeRF dataset. VeGaS reconstructs fine-grained structural details more faithfully than competing methods, such as the vertical armor ridges in the \textit{hook} scene and the arm structures in the \textit{mutant} scene. These results demonstrate the superior synthesis quality of VeGaS, particularly in preserving local geometry and detailed dynamic appearance.
%Fig.\ref{synthetic_quantitative_results} illustrates qualitative comparisons on the D-NeRF dataset. VeGaS accurately reconstructs fine-grained structural details, including the vertical ridges of the armor in the \textit{hook} scene and the detailed arm structures in the \textit{mutant} scene, demonstrating noticeably superior synthesis quality compared to competing approaches.

\subsection{Ablation Studies}
Extensive ablation studies are conducted on Neu3DV set to validate the effectiveness of our method.

\vspace{-1.5mm}
\textbf{Effectiveness of velocity and geometric modeling.}
We analyze the contributions of the proposed velocity and geometric modeling components by evaluating two ablated variants. First, we retain only the time-varying velocity component to assess its effect on motion modeling. As shown in Fig.~\ref{visual_ablation} on the \textit{sear\_steak} scene, velocity modeling substantially improves the reconstruction of rigid object motion. It better captures the trajectory of the meat clamp, producing sharper and more temporally consistent object contours than the variant with geometric modeling only. Second, we retain only the Geometric Deformation Network while disabling the velocity component. As illustrated in the \textit{flame\_steak} scene, geometric modeling is particularly beneficial for highly deformable objects such as flames, leading to more faithful shape variations and improved reconstruction quality. These results show that velocity modeling and geometric modeling capture complementary dynamic cues: the former focuses on coherent motion trajectories, while the latter enhances deformation representation.

%\textbf{Effectiveness of Velocity modeling.} To assess the impact of velocity modeling, we incorporate only the proposed time-varying velocity. As shown in Fig.~\ref{visual_ablation} for the \textit{sear\_steak} scene, introducing velocity modeling significantly improves the reconstruction of rigid object motion. It more accurately captures the motion trajectory of the meat clamp, resulting in clearer and more consistent object contours compared to the version utilizing only geometry modeling.

\begin{table}[t!]
\centering
% \setlength{\tabcolsep}{4.0pt}
% \small
\caption{Ablation study on the components of our design.}
\begin{tabular}{l|ccccc}
\toprule
\multicolumn{1}{l|}{\textbf{Methods}} & \textbf{PSNR} $\uparrow$ & \textbf{SSIM} $\uparrow$ & \textbf{LPIPS} $\downarrow$ & \textbf{FPS} $\uparrow$ & \textbf{Storage (MB)} $\downarrow$\\
\midrule
4DGS \cite{yang2023real}  & 32.01 & 0.97 & 0.055 & 114 & 2195\\ 
+ velocity                    & 32.07 & 0.97 & 0.050 & 202 & 1124\\  
+ geometric modeling          & 32.17 & 0.97 & 0.052 & 134 & 1088\\
\midrule
VeGaS (Ours Full)             & \textbf{32.68} & \textbf{0.98} & \textbf{0.045} & \textbf{115} & \textbf{1124}\\
\bottomrule
\end{tabular}
\label{tab:Ablation1}
\vspace{-6mm}
\end{table}

\begin{table}[t!]
\centering
\caption{Ablation study on the impact of anchor number.}

\setlength{\tabcolsep}{3.5pt} 

\resizebox{\linewidth}{!}{
    \begin{tabular}{c|c|ccc | c|c|ccc}
    \toprule
    
    \multirow{6}{*}{\rotatebox{90}{flame\_salmon}} 
    & \textbf{\# Anchors} & \textbf{PSNR} $\uparrow$ & \textbf{SSIM} $\uparrow$ & \textbf{LPIPS} $\downarrow$ 
    & \multirow{6}{*}{\rotatebox{90}{trex}} 
    & \textbf{\# Anchors} & \textbf{PSNR} $\uparrow$ & \textbf{SSIM} $\uparrow$ & \textbf{LPIPS} $\downarrow$ \\
    
    \cmidrule(lr){2-5} \cmidrule(lr){7-10}
    
    & 2          & 29.25 & 0.960 & 0.065   &  & 2          & 31.30 & 0.993 & 0.011\\
    & 4          & 29.54 & 0.962 & 0.063   &  & 4          & 31.43 & 0.993 & 0.011\\
    & 6   & \textbf{29.88} & \textbf{0.964} & \textbf{0.061}   &  & 6   & \textbf{31.51} & \textbf{0.993} & \textbf{0.011}\\
    & 8          & 29.30 & 0.961 & 0.062   &  & 8          & 31.34 & 0.993 & 0.011\\
    & 10         & 29.15 & 0.955 & 0.072   &  & 10         & 31.26 & 0.992 & 0.012\\
    
    \bottomrule
    \end{tabular}
}
\label{tab:Ablation2}
\vspace{-3mm}
\end{table}

\begin{figure}[t!]
\centerline{\includegraphics[width=0.92\textwidth]{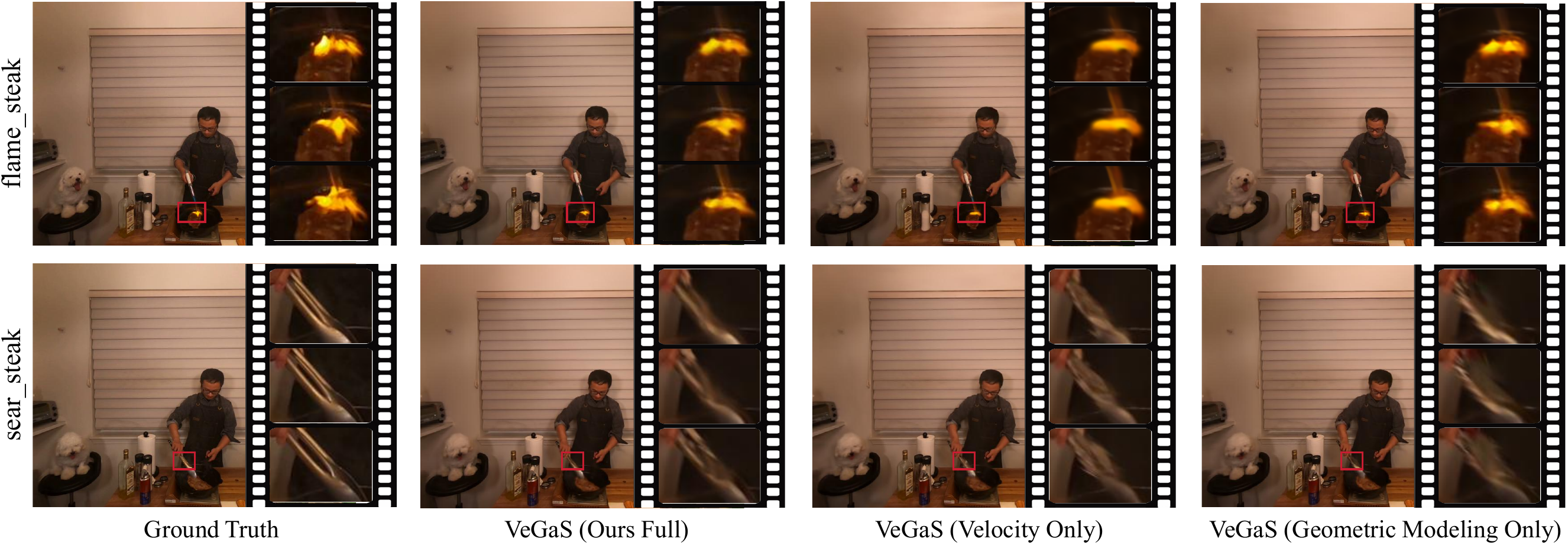}}
\caption{Visualization of ablation studies on the Neu3DV dataset, where we individually remove velocity and geometric modeling to evaluate their impact on the model's performance. Continuous video frames are extracted to observe the effects of removing these components.}
\label{visual_ablation}
\vspace{-6mm}
\end{figure}

%\textbf{Effectiveness of Geometric Modeling.} We further analyze the contribution of geometric modeling by only introducing the proposed Geometric Deformation Network, while disabling the velocity component. As illustrated in Fig.~\ref{visual_ablation} for the \textit{flame\_steak} scene, this geometric enhancement substantially improves the reconstruction quality of highly deformable objects such as flames, leading to more faithful shape variations.

% \textbf{Full Method.} By jointly incorporating time-varying velocity and the Geometric Deformation Network, the complete VeGaS framework achieves a significant improvement over the original 4DGS baseline, as reported in Tab.~\ref{tab:Ablation1}. These results indicate that effective dynamic scene reconstruction benefits from both flexible motion modeling and accurate geometric deformation modeling.
\vspace{-1.5mm}
\textbf{Comprehensive comparison with 4DGS.}
By jointly incorporating time-varying velocity and the Geometric Deformation Network, the complete VeGaS framework achieves substantial improvements over 4DGS, as reported in Tab.~\ref{tab:Ablation1}. In addition, the proposed design improves the expressive capacity of individual Gaussians with negligible computational overhead, enabling high-quality reconstruction with significantly fewer primitives. Specifically, VeGaS reduces the number of Gaussians from 3572K to 1690K and saves approximately 1000 MB of storage, while maintaining a rendering speed comparable to the 4DGS. These results indicate that effective dynamic scene reconstruction benefits from both coherent motion modeling and accurate geometric deformation modeling.
%\textbf{Full Method.} By jointly incorporating time-varying velocity and the Geometric Deformation Network, the complete VeGaS framework achieves a significant improvement over the original 4DGS baseline, as reported in Tab.~\ref{tab:Ablation1}. Furthermore, our method improves the expressive capacity of each Gaussian with negligible overhead, enabling strong reconstruction quality with substantially fewer Gaussians. As a result, the number of Gaussians is reduced from 3572K to 1690K, saving about 1000 MB of storage while maintaining a rendering speed comparable to the baseline. These results indicate that effective dynamic scene reconstruction benefits from both flexible motion modeling and accurate geometric deformation modeling.

\vspace{-1.5mm}
\textbf{Impact of anchor number for velocity computation.} 
We further study the effect of the number of velocity anchors on two relatively complex scenes, as reported in Tab.~\ref{tab:Ablation2}. Using 6 anchors achieves the best overall performance, yielding the highest PSNR and SSIM and the lowest LPIPS on \textit{flame\_salmon}, as well as the highest PSNR on \textit{trex}. Increasing the number of anchors beyond 6 does not bring further improvement and may even degrade reconstruction quality, suggesting that excessive anchors can introduce redundant or unstable motion bases. These results indicate that 6 anchors provide a favorable balance between representation capacity and optimization stability, and are sufficient to capture complex scene motions.
%We further conduct an ablation study on the number of velocity anchors using two relatively complex scenes, as reported in Tab.~\ref{tab:Ablation2}. The results indicate that utilizing 6 anchors provides the best overall trade-off between reconstruction quality and efficiency, which is already sufficient to capture the complex motion across scenes. 
% Although this anchor number in practice can be chosen empirically based on specific scene characteristics, we fix it to 6 across all experiments to maintain a unified protocol and ensure fair comparison.

\section{Conclusion}
We introduce VeGaS, a velocity-aware framework that decouples motion and geometry in 4D Gaussian Splatting. VeGaS overcomes a core limitation of prior methods, which conflate motion and geometric attributes within a unified covariance formulation. Drawing inspiration from Galilean shearing, we parameterize motion via a time-varying velocity, enabling flexible modeling of non-linear dynamics while maintaining a geometry-preserving conditional covariance. In addition, a geometric deformation network is employed to enhance temporal geometric expressiveness. Extensive experiments on real-world and synthetic benchmarks validate the effectiveness of the proposed design.

% \begin{ack}
% Use unnumbered first level headings for the acknowledgments. All acknowledgments

% Do {\bf not} include this section in the anonymized submission, only in the final paper. You can use the \texttt{ack} environment provided in the style file to automatically hide this section in the anonymized submission.
% \end{ack}

\bibliography{main}
\bibliographystyle{plain}
% \section*{References}

% References follow the acknowledgments in the camera-ready paper. Use unnumbered first-level heading for
% the references. Any choice of citation style is acceptable as long as you are
% consistent. It is permissible to reduce the font size to \verb+small+ (9 point)
% when listing the references.
% Note that the Reference section does not count towards the page limit.
% \medskip

% {
% \small

% [1] Alexander, J.A.\ \& Mozer, M.C.\ (1995) Template-based algorithms for
% connectionist rule extraction. In G.\ Tesauro, D.S.\ Touretzky and T.K.\ Leen
% (eds.), {\it Advances in Neural Information Processing Systems 7},
% pp.\ 609--616. Cambridge, MA: MIT Press.

% [2] Bower, J.M.\ \& Beeman, D.\ (1995) {\it The Book of GENESIS: Exploring
%   Realistic Neural Models with the GEneral NEural SImulation System.}  New York:
% TELOS/Springer--Verlag.

% [3] Hasselmo, M.E., Schnell, E.\ \& Barkai, E.\ (1995) Dynamics of learning and
% recall at excitatory recurrent synapses and cholinergic modulation in rat
% hippocampal region CA3. {\it Journal of Neuroscience} {\bf 15}(7):5249-5262.
% }

%%%%%%%%%%%%%%%%%%%%%%%%%%%%%%%%%%%%%%%%%%%%%%%%%%%%%%%%%%%%
\newpage
\appendix

\section{Additional Experiment Results}

\begin{figure}[h]
\centerline{\includegraphics[width=\textwidth]{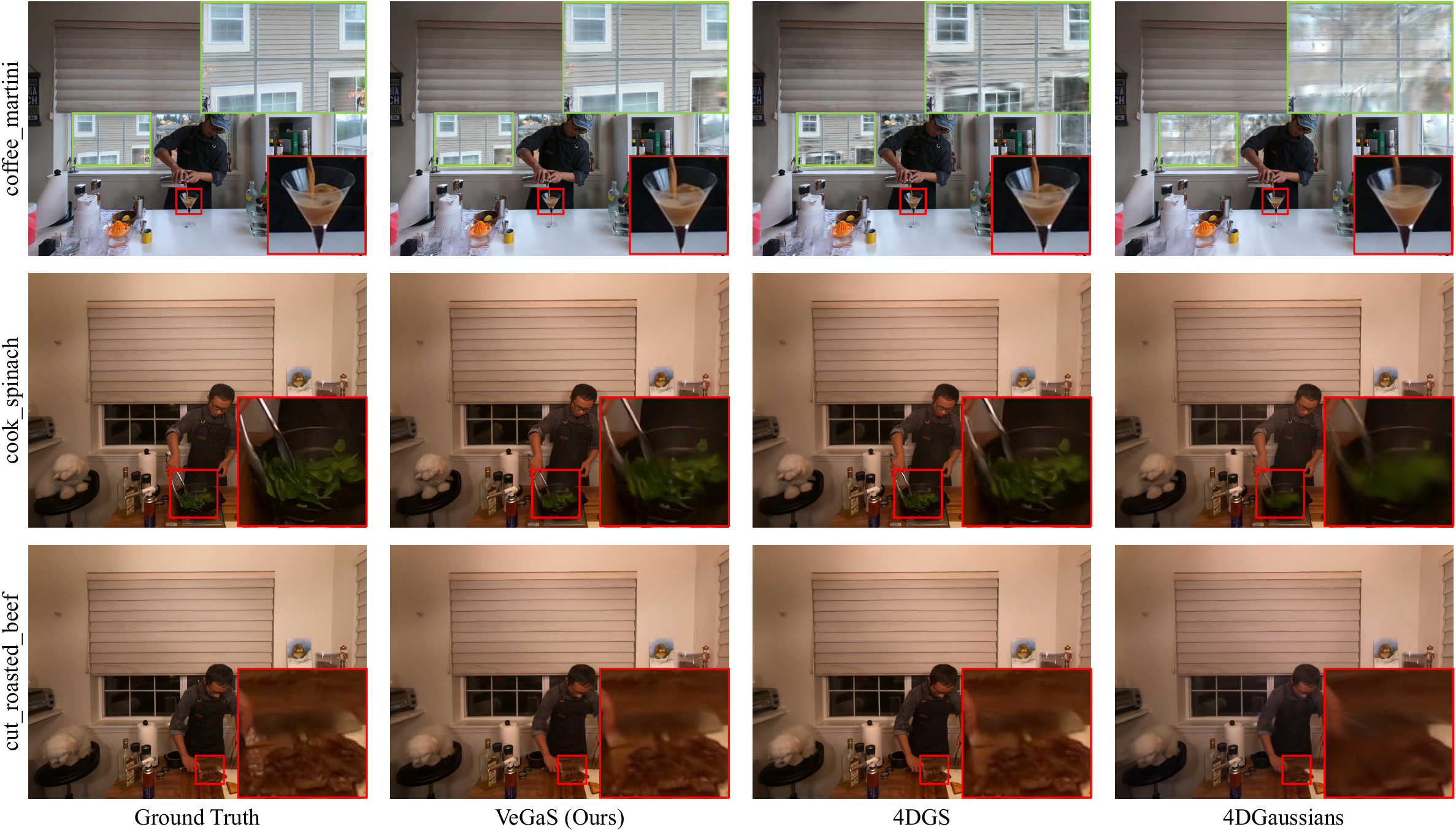}}
\caption{Additional qualitative result on the Neural 3D Video dataset. Our method exhibits noticeably higher visual quality compared to others.}
\end{figure}

\section{Extended Discussions and Limitations}

\textbf{Q:} How does the fixed number of velocity anchors ($N_v=6$) generalize to long-duration videos or high-frequency motions?

\textbf{A:} Using a fixed number of velocity anchors $N_v=6$ introduces a structural limitation and may become suboptimal for very long videos or highly unpredictable high-frequency motion. However, in our experiments on Neu3DV and D-NeRF, $N_v=6$ provides the best overall trade-off between reconstruction quality and efficiency. Although it may not be the optimal choice for every single scene, we fix $N_v=6$ across all experiments to maintain a unified protocol and ensure fair comparison. At the same time, the anchor number in practice can be chosen empirically based on scene characteristics, such as video complexity or cues from COLMAP reconstruction. This choice is not fundamental to our framework, which is fully compatible with more adaptive strategies such as dynamic anchor allocation.

\section{Detailed derivation of the conditional distribution of a 4D Gaussian}
\label{appendix:derivation_conditional_distribution}

\subsection{Transformation of the 4D Covariance Matrix}
Given the Galilean shearing matrix $V$, the modified covariance matrix is given by:
\begin{equation}
\Sigma' = V R S S^T R^T V^T = V \Sigma V^T,
\end{equation}
where $R$ denotes the rotation matrix and $S$ denotes the scaling matrix. The shearing velocity matrix $\mathbf{V}$ explicitly represents the positional transformation of a Gaussian point in the 4D spatio-temporal continuum.

By expanding the block matrix multiplication, we can derive the structure of the covariance matrix $\Sigma'$:
\begin{equation}
\Sigma' =
\begin{pmatrix}
\mathbf{I}_3 & v(t) \\
\mathbf{0} & 1
\end{pmatrix}
\begin{pmatrix}
\Sigma_{1:3,\,1:3} & \Sigma_{1:3,\,4} \\
\Sigma_{4,\,1:3} & \Sigma_{4,\,4}
\end{pmatrix}
\begin{pmatrix}
\mathbf{I}_3 & \mathbf{0} \\
v(t)^\top & 1
\end{pmatrix}.
\end{equation}

Carrying out the multiplication for the block matrix $\Sigma = \begin{pmatrix}
\Sigma_{1:3,\,1:3} & \Sigma_{1:3,\,4} \\
\Sigma_{4,\,1:3} & \Sigma_{4,\,4}
\end{pmatrix}$, the individual blocks of the updated covariance matrix $\Sigma'$ are obtained as follows:
\begin{equation}
\label{equation_each_block}
\begin{aligned}
\quad \Sigma'_{1:3,\,1:3} =& \Sigma_{1:3,\,1:3} + v(t)\Sigma_{4,\,1:3} + \Sigma_{1:3,\,4}v(t)^\top \\ &+ v(t)\Sigma_{4,\,4}v(t)^\top, \\
\Sigma'_{1:3,\,4} =& \Sigma_{1:3,\,4} + v(t)\Sigma_{4,\,4}, \\
\Sigma'_{4,\,1:3} =& \Sigma_{4,\,1:3} + \Sigma_{4,\,4}v(t)^\top, \\
\Sigma'_{4,\,4} =& \Sigma_{4,\,4}.
\end{aligned}
\end{equation}

\subsection{General Multivariate Conditional Gaussian Distribution}
Before applying the conditional logic to the 4D Gaussian, we first derive the general form of the conditional distribution. 

Consider a multivariate Gaussian random variable $\mathbf{x} \sim \mathcal{N}(\boldsymbol{\mu}, \boldsymbol{\Sigma})$ partitioned into two subsets $\mathbf{x}_a$ and $\mathbf{x}_b$:
\begin{equation}
\mathbf{x} = \begin{pmatrix} \mathbf{x}_a \\ \mathbf{x}_b \end{pmatrix}, \quad
\boldsymbol{\mu} = \begin{pmatrix} \boldsymbol{\mu}_a \\ \boldsymbol{\mu}_b \end{pmatrix}, \quad
\boldsymbol{\Sigma} = \begin{pmatrix} \boldsymbol{\Sigma}_{aa} & \boldsymbol{\Sigma}_{ab} \\ \boldsymbol{\Sigma}_{ba} & \boldsymbol{\Sigma}_{bb} \end{pmatrix}.
\end{equation}
To find the conditional distribution $p(\mathbf{x}_a | \mathbf{x}_b)$, we construct a linear transformation to decorrelate $\mathbf{x}_a$ from $\mathbf{x}_b$. Let $\mathbf{z} = \mathbf{x}_a - \mathbf{A}\mathbf{x}_b$. We seek a matrix $\mathbf{A}$ such that $\mathbf{z}$ and $\mathbf{x}_b$ are uncorrelated (and thus independent for Gaussians):
\begin{equation}
\text{Cov}(\mathbf{z}, \mathbf{x}_b) = \text{Cov}(\mathbf{x}_a, \mathbf{x}_b) - \mathbf{A}\text{Cov}(\mathbf{x}_b, \mathbf{x}_b) = \boldsymbol{\Sigma}_{ab} - \mathbf{A}\boldsymbol{\Sigma}_{bb} = \mathbf{0}.
\end{equation}
Solving for $\mathbf{A}$ yields $\mathbf{A} = \boldsymbol{\Sigma}_{ab}\boldsymbol{\Sigma}_{bb}^{-1}$. Since $\mathbf{z}$ is independent of $\mathbf{x}_b$, the conditional expectation is derived as:
\begin{align}
\mathbb{E}[\mathbf{x}_a | \mathbf{x}_b] &= \mathbb{E}[\mathbf{A}\mathbf{x}_b + \mathbf{z} | \mathbf{x}_b] = \mathbf{A}\mathbf{x}_b + \mathbb{E}[\mathbf{z}] \nonumber \\
&= \boldsymbol{\Sigma}_{ab}\boldsymbol{\Sigma}_{bb}^{-1}\mathbf{x}_b + (\boldsymbol{\mu}_a - \boldsymbol{\Sigma}_{ab}\boldsymbol{\Sigma}_{bb}^{-1}\boldsymbol{\mu}_b) \nonumber \\
&= \boldsymbol{\mu}_a + \boldsymbol{\Sigma}_{ab}\boldsymbol{\Sigma}_{bb}^{-1}(\mathbf{x}_b - \boldsymbol{\mu}_b).
\end{align}
The conditional covariance is simply the variance of the residual $\mathbf{z}$:
\begin{equation}
\text{Cov}(\mathbf{x}_a | \mathbf{x}_b) = \text{Cov}(\mathbf{z}) = \boldsymbol{\Sigma}_{aa} - \boldsymbol{\Sigma}_{ab}\boldsymbol{\Sigma}_{bb}^{-1}\boldsymbol{\Sigma}_{ba}.
\end{equation}

\subsection{Application to Our 4D Gaussian Representation}
Applying the general derivation above to our specific 4D case, we map the spatial coordinates to partition $a$ (indices $1:3$) and the temporal coordinate to partition $b$ (index $4$). We substitute the transformed covariance blocks $\Sigma'$ into the conditional formulas.

The modified conditional mean $\mu'_{xyz|t}$ is given by:
\begin{align}
\mu'_{xyz|t} &= \mu_{1:3} + \Sigma'_{1:3,4} (\Sigma'_{4,4})^{-1} (t - \mu_t) \nonumber \\
&= \mu_{1:3} + \left( \Sigma_{1:3,4} + v(t)\Sigma_{4,4} \right) \Sigma_{4,4}^{-1} (t - \mu_t) \nonumber \\
&= \mu_{1:3} + \left( \Sigma_{1:3,4}\Sigma_{4,4}^{-1} + v(t) \right) (t - \mu_t).
\end{align}

Similarly, the modified conditional covariance $\Sigma'_{xyz|t}$ is calculated as:
\begin{equation}
\Sigma'_{xyz|t} = \Sigma'_{1:3,1:3} - \Sigma'_{1:3,4}(\Sigma'_{4,4})^{-1}\Sigma'_{4,1:3}.
\end{equation}
To simplify this, we first expand the subtraction term $K = \Sigma'_{1:3,4}(\Sigma'_{4,4})^{-1}\Sigma'_{4,1:3}$:
\begin{align}
K &= \left( \Sigma_{1:3,4} + v(t)\Sigma_{4,4} \right) \Sigma_{4,4}^{-1} \left( \Sigma_{4,1:3} + \Sigma_{4,4}v(t)^\top \right) \nonumber \\
&= \left( \Sigma_{1:3,4}\Sigma_{4,4}^{-1} + v(t) \right) \left( \Sigma_{4,1:3} + \Sigma_{4,4}v(t)^\top \right) \nonumber \\
&= \Sigma_{1:3,4}\Sigma_{4,4}^{-1}\Sigma_{4,1:3} + \Sigma_{1:3,4}v(t)^\top + v(t)\Sigma_{4,1:3} + v(t)\Sigma_{4,4}v(t)^\top.
\end{align}
Substituting $K$ and the expression for $\Sigma'_{1:3,1:3}$ from Eq.~\eqref{equation_each_block} back into the conditional covariance formula, we observe that all terms involving $v(t)$ cancel out:
\begin{align}
\Sigma'_{xyz|t} &= \left( \Sigma_{1:3,1:3} + v(t)\Sigma_{4,1:3} + \Sigma_{1:3,4}v(t)^\top + v(t)\Sigma_{4,4}v(t)^\top \right) - K \nonumber \\
&= \Sigma_{1:3,1:3} - \Sigma_{1:3,4}\Sigma_{4,4}^{-1}\Sigma_{4,1:3}.
\end{align}
This proves that the shape of the conditional covariance remains invariant under the Galilean shear transformation, identical to the original spatial covariance.

\section{Schur Complement}
\begin{definition}
[Schur Complement]
Schur Complement]Consider a partitioned symmetric matrix $\mathbf{M} \in \mathbb{R}^{(n+m) \times (n+m)}$, decomposed into block matrices as follows:
\begin{equation}
\mathbf{M} = \begin{pmatrix} 
\mathbf{A} & \mathbf{B} \\ 
\mathbf{B}^\top & \mathbf{D} 
\end{pmatrix},
\end{equation}
where $\mathbf{A} \in \mathbb{R}^{n \times n}$, $\mathbf{B} \in \mathbb{R}^{n \times m}$, and $\mathbf{D} \in \mathbb{R}^{m \times m}$ is invertible. The \textit{Schur complement} of the block $\mathbf{D}$ in $\mathbf{M}$, denoted as $\mathbf{M} / \mathbf{D}$, is defined as:
\begin{equation}
\mathbf{M} / \mathbf{D} \triangleq \mathbf{A} - \mathbf{B}\mathbf{D}^{-1}\mathbf{B}^\top.
\end{equation}
\end{definition}

\begin{remark}
[Probabilistic Interpretation]
In the context of multivariate Gaussian distributions, if $\mathbf{M}$ represents the joint covariance matrix of two random vectors $\mathbf{x} \in \mathbb{R}^n$ and $\mathbf{y} \in \mathbb{R}^m$, the Schur complement $\mathbf{M} / \mathbf{D}$ corresponds precisely to the \textit{conditional covariance} of $\mathbf{x}$ given $\mathbf{y}$ (i.e., $\text{Cov}(\mathbf{x} | \mathbf{y})$). 
\end{remark}

\section{Proof of Schur Complement Invariance}
\label{appendix:proof_schur}

\begin{theorem}
[Schur Complement Invariance] 
Let $\mathbf{\Sigma} \in \mathbb{R}^{4 \times 4}$ be a symmetric positive semi-definite covariance matrix of a 4D Gaussian, and let $\mathbf{\Sigma}' = \mathbf{V}\mathbf{\Sigma}\mathbf{V}^\top$ be the congruence transformation induced by the shearing matrix $\mathbf{V}$. Denoting by $\text{Schur}_{4,4}(\cdot)$ the Schur complement with respect to the temporal dimension, the following invariance holds:
\begin{equation}
\text{Schur}_{4,4}(\mathbf{\Sigma}') = \text{Schur}_{4,4}(\mathbf{\Sigma}).
\end{equation}
\end{theorem}

\begin{proof}
We partition the covariance matrix $\mathbf{\Sigma}$ and the shearing matrix $\mathbf{V}$ into block forms corresponding to the spatial ($1:3$) and temporal ($4$) dimensions:
\begin{equation}
\mathbf{\Sigma} = \begin{pmatrix}
\mathbf{\Sigma}_{xx} & \mathbf{\Sigma}_{xt} \\
\mathbf{\Sigma}_{tx} & \Sigma_{tt}
\end{pmatrix}, \quad
\mathbf{V} = \begin{pmatrix}
\mathbf{I}_3 & \mathbf{v} \\
\mathbf{0}^\top & 1
\end{pmatrix},
\end{equation}
where $\mathbf{\Sigma}_{xx} \in \mathbb{R}^{3 \times 3}$, $\mathbf{\Sigma}_{xt} \in \mathbb{R}^{3 \times 1}$, $\Sigma_{tt} \in \mathbb{R}^{1 \times 1}$, and $\mathbf{v} \in \mathbb{R}^{3 \times 1}$ is the velocity vector. Note that $\mathbf{\Sigma}_{tx} = \mathbf{\Sigma}_{xt}^\top$.

First, we compute the block structure of the transformed covariance $\mathbf{\Sigma}' = \mathbf{V}\mathbf{\Sigma}\mathbf{V}^\top$. Expanding the matrix multiplication yields:
\begin{align}
\mathbf{\Sigma}' &= \begin{pmatrix}
\mathbf{I}_3 & \mathbf{v} \\
\mathbf{0} & 1
\end{pmatrix}
\begin{pmatrix}
\mathbf{\Sigma}_{xx} & \mathbf{\Sigma}_{xt} \\
\mathbf{\Sigma}_{tx} & \Sigma_{tt}
\end{pmatrix}
\begin{pmatrix}
\mathbf{I}_3 & \mathbf{0} \\
\mathbf{v}^\top & 1
\end{pmatrix} \nonumber \\
&= \begin{pmatrix}
\mathbf{\Sigma}_{xx} + \mathbf{v}\mathbf{\Sigma}_{tx} & \mathbf{\Sigma}_{xt} + \mathbf{v}\Sigma_{tt} \\
\mathbf{\Sigma}_{tx} & \Sigma_{tt}
\end{pmatrix}
\begin{pmatrix}
\mathbf{I}_3 & \mathbf{0} \\
\mathbf{v}^\top & 1
\end{pmatrix} \nonumber \\
&= \begin{pmatrix}
\mathbf{\Sigma}_{xx} + \mathbf{v}\mathbf{\Sigma}_{tx} + (\mathbf{\Sigma}_{xt} + \mathbf{v}\Sigma_{tt})\mathbf{v}^\top & \mathbf{\Sigma}_{xt} + \mathbf{v}\Sigma_{tt} \\
\mathbf{\Sigma}_{tx} + \Sigma_{tt}\mathbf{v}^\top & \Sigma_{tt}
\end{pmatrix}.
\end{align}
From the expansion above, we identify the individual blocks of $\mathbf{\Sigma}'$:
\begin{subequations}
\begin{align}
\mathbf{\Sigma}'_{xx} &= \mathbf{\Sigma}_{xx} + \mathbf{v}\mathbf{\Sigma}_{tx} + \mathbf{\Sigma}_{xt}\mathbf{v}^\top + \mathbf{v}\Sigma_{tt}\mathbf{v}^\top, \label{eq:sigma_xx_prime} \\
\mathbf{\Sigma}'_{xt} &= \mathbf{\Sigma}_{xt} + \mathbf{v}\Sigma_{tt}, \label{eq:sigma_xt_prime} \\
\Sigma'_{tt} &= \Sigma_{tt}. \label{eq:sigma_tt_prime}
\end{align}
\end{subequations}

The Schur complement of the transformed matrix, denoted as $S'$, is defined by:
\begin{equation}
S' = \text{Schur}_{4,4}(\mathbf{\Sigma}') = \mathbf{\Sigma}'_{xx} - \mathbf{\Sigma}'_{xt}(\Sigma'_{tt})^{-1}(\mathbf{\Sigma}'_{xt})^\top.
\end{equation}

Substituting the expressions from Eqs.~\eqref{eq:sigma_xx_prime}--\eqref{eq:sigma_tt_prime} into the definition, we focus on expanding the subtraction term $K = \mathbf{\Sigma}'_{xt}(\Sigma'_{tt})^{-1}(\mathbf{\Sigma}'_{xt})^\top$:
\begin{align}
K &= (\mathbf{\Sigma}_{xt} + \mathbf{v}\Sigma_{tt}) \Sigma_{tt}^{-1} (\mathbf{\Sigma}_{tx} + \Sigma_{tt}\mathbf{v}^\top) \nonumber \\
&= (\mathbf{\Sigma}_{xt}\Sigma_{tt}^{-1} + \mathbf{v}) (\mathbf{\Sigma}_{tx} + \Sigma_{tt}\mathbf{v}^\top) \nonumber \\
&= \mathbf{\Sigma}_{xt}\Sigma_{tt}^{-1}\mathbf{\Sigma}_{tx} + \mathbf{\Sigma}_{xt}\mathbf{v}^\top + \mathbf{v}\mathbf{\Sigma}_{tx} + \mathbf{v}\Sigma_{tt}\mathbf{v}^\top.
\end{align}
Now, we subtract $K$ from $\mathbf{\Sigma}'_{xx}$:
\begin{align}
S' &= \mathbf{\Sigma}'_{xx} - K \nonumber \\
&= \left( \mathbf{\Sigma}_{xx} + \mathbf{v}\mathbf{\Sigma}_{tx} + \mathbf{\Sigma}_{xt}\mathbf{v}^\top + \mathbf{v}\Sigma_{tt}\mathbf{v}^\top \right) \nonumber \\
&\quad - \left( \mathbf{\Sigma}_{xt}\Sigma_{tt}^{-1}\mathbf{\Sigma}_{tx} + \mathbf{\Sigma}_{xt}\mathbf{v}^\top + \mathbf{v}\mathbf{\Sigma}_{tx} + \mathbf{v}\Sigma_{tt}\mathbf{v}^\top \right).
\end{align}
Observing the terms, we see that $\mathbf{v}\mathbf{\Sigma}_{tx}$, $\mathbf{\Sigma}_{xt}\mathbf{v}^\top$, and $\mathbf{v}\Sigma_{tt}\mathbf{v}^\top$ appear in both parts and cancel out perfectly. The remaining terms are:
\begin{equation}
S' = \mathbf{\Sigma}_{xx} - \mathbf{\Sigma}_{xt}\Sigma_{tt}^{-1}\mathbf{\Sigma}_{tx} = \text{Schur}_{4,4}(\mathbf{\Sigma}).
\end{equation}
Thus, the Schur complement is invariant under the Galilean shear transformation.
\end{proof}

\section{Generalization of Constant-Velocity Galilean Shearing to Time-Varying Velocity}
\label{appendix:time_varying_velocity}
While the standard Galilean transformation is defined for inertial frames with constant velocity, real-world dynamics often involve non-linear trajectories with time-varying velocities. Here, we provide the mathematical justification for extending the shearing mechanism to model such motions using \textit{local linearization}.

\textbf{First-Order Taylor Approximation.}
Consider a 4D Gaussian centered at temporal coordinate $\mu_t$, tracking a particle moving along a non-linear trajectory $\gamma(t) \in \mathbb{R}^3$. We aim to approximate this trajectory within the local temporal neighborhood of the Gaussian, defined effectively by its temporal variance $\Sigma_{tt}$.

Expanding the trajectory $\gamma(t)$ around the center time $\mu_t$ using a Taylor series, we obtain:
\begin{equation}
\gamma(t) = \gamma(\mu_t) + \frac{d\gamma}{dt}\bigg|_{t=\mu_t} (t - \mu_t) + \mathcal{O}\left((t - \mu_t)^2\right).
\end{equation}
Let $\mathbf{v} = \frac{d\gamma}{dt}|_{t=\mu_t}$ denote the instantaneous velocity at the center of the Gaussian. Neglecting higher-order terms $\mathcal{O}((t - \mu_t)^2)$ (which is a valid assumption for Gaussians with small temporal extent), the trajectory is approximated as a linear function:
\begin{equation}
\label{eq:linear_approx}
\hat{\gamma}(t) \approx \gamma(\mu_t) + \mathbf{v} \cdot (t - \mu_t).
\end{equation}

\textbf{Equivalence to Galilean Shearing.}
Recall the operation of the Galilean shearing matrix $\mathbf{V}$ defined in the main text on a spatial point $\mathbf{x}$ relative to the temporal center:
\begin{equation}
\begin{pmatrix} \mathbf{x}' \\ t' \end{pmatrix} = 
\begin{pmatrix} \mathbf{I}_3 & \mathbf{v} \\ \mathbf{0} & 1 \end{pmatrix} 
\begin{pmatrix} \mathbf{x} \\ t - \mu_t \end{pmatrix} 
= \begin{pmatrix} \mathbf{x} + \mathbf{v}(t - \mu_t) \\ t - \mu_t \end{pmatrix}.
\end{equation}
By setting the initial position $\mathbf{x} = \gamma(\mu_t)$, the spatial component becomes $\mathbf{x}' = \gamma(\mu_t) + \mathbf{v}(t - \mu_t)$, which is identical to the first-order approximation in Eq.~\eqref{eq:linear_approx}.

Thus, applying the shearing matrix $\mathbf{V}$ parameterized by the instantaneous velocity $\mathbf{v}(t)$ is mathematically equivalent to locally linearizing the non-linear trajectory along the tangent direction at $t=\mu_t$.

\section{Preservation of Symmetric Positive Semi-Definiteness}
\label{appendix:spsd_proof}
A valid covariance matrix must be symmetric and positive semi-definite (SPSD) to represent a meaningful probability distribution. We prove that the congruence transformation induced by the Galilean shearing matrix preserves these essential properties.

\begin{lemma}
[Invariance of SPSD under Congruence Transformation.]
Let $\mathbf{\Sigma} \in \mathbb{R}^{4 \times 4}$ be a symmetric positive semi-definite matrix ($\mathbf{\Sigma} \succeq 0$), and let $\mathbf{V} \in \mathbb{R}^{4 \times 4}$ be the Galilean shearing matrix. The transformed covariance matrix $\mathbf{\Sigma}' = \mathbf{V}\mathbf{\Sigma}\mathbf{V}^\top$ remains symmetric and positive semi-definite.
\end{lemma}

\begin{proof}
We verify these properties separately:

\textbf{Symmetry.}
By definition, $\mathbf{\Sigma}$ is symmetric, so $\mathbf{\Sigma} = \mathbf{\Sigma}^\top$. Taking the transpose of the transformed matrix $\mathbf{\Sigma}'$:
\begin{align}
(\mathbf{\Sigma}')^\top &= (\mathbf{V}\mathbf{\Sigma}\mathbf{V}^\top)^\top \nonumber \\
&= (\mathbf{V}^\top)^\top \mathbf{\Sigma}^\top \mathbf{V}^\top \nonumber \\
&= \mathbf{V} \mathbf{\Sigma} \mathbf{V}^\top \nonumber \\
&= \mathbf{\Sigma}'.
\end{align}
Thus, $\mathbf{\Sigma}'$ is symmetric.

\textbf{Positive Semi-Definiteness.}
By definition, $\mathbf{\Sigma} \succeq 0$ implies that for any non-zero vector $\mathbf{x} \in \mathbb{R}^4$, the quadratic form satisfies $\mathbf{x}^\top \mathbf{\Sigma} \mathbf{x} \ge 0$.
Consider the quadratic form of the transformed matrix $\mathbf{\Sigma}'$ with respect to an arbitrary vector $\mathbf{y} \in \mathbb{R}^4$:
\begin{equation}
Q(\mathbf{y}) = \mathbf{y}^\top \mathbf{\Sigma}' \mathbf{y} = \mathbf{y}^\top (\mathbf{V}\mathbf{\Sigma}\mathbf{V}^\top) \mathbf{y}.
\end{equation}
Using the associative property of matrix multiplication, we can regroup the terms:
\begin{equation}
Q(\mathbf{y}) = (\mathbf{y}^\top \mathbf{V}) \mathbf{\Sigma} (\mathbf{V}^\top \mathbf{y}).
\end{equation}
Let $\mathbf{z} = \mathbf{V}^\top \mathbf{y}$. The equation becomes:
\begin{equation}
Q(\mathbf{y}) = \mathbf{z}^\top \mathbf{\Sigma} \mathbf{z}.
\end{equation}
Since $\mathbf{\Sigma}$ is positive semi-definite, $\mathbf{z}^\top \mathbf{\Sigma} \mathbf{z} \ge 0$ holds for any vector $\mathbf{z}$. Consequently, $\mathbf{y}^\top \mathbf{\Sigma}' \mathbf{y} \ge 0$ for all $\mathbf{y}$.

Therefore, $\mathbf{\Sigma}'$ is positive semi-definite.
\end{proof}

\begin{corollary}
[Preservation of Positive Definiteness]
\label{cor:positive_definite}
Furthermore, if $\mathbf{\Sigma}$ is strictly positive definite ($\mathbf{\Sigma} \succ 0$) and $\mathbf{V}$ is non-singular, then $\mathbf{\Sigma}'$ is also strictly positive definite.
For the Galilean shearing matrix, the determinant is:
\begin{equation}
\det(\mathbf{V}) = 1 \neq 0.
\end{equation}
Since $\mathbf{V}$ is invertible (full rank), the transformation strictly preserves the positive definiteness of the Gaussian covariance.
\end{corollary}

\section{Impact Statement}
This paper presents work whose goal is to advance the field of Machine Learning, which focuses on practical novel view synthesis. There are many potential societal consequences of our work, none of which we feel must be specifically highlighted here.
%%%%%%%%%%%%%%%%%%%%%%%%%%%%%%%%%%%%%%%%%%%%%%%%%%%%%%%%%%%%

% \newpage
% \input{checklist.tex}

\end{document}